\documentclass[11pt]{article}

\usepackage[utf8]{inputenc} 
\usepackage[T1]{fontenc}    
\usepackage{lmodern}
\usepackage{hyperref}       
\usepackage{url}            
\usepackage{booktabs}       
\usepackage{amsfonts}       
\usepackage{nicefrac}       
\usepackage{microtype}      
\usepackage{csquotes}
\usepackage{float}
\usepackage{multirow}

\usepackage{color}
\usepackage{xifthen}
\usepackage{enumitem}
\usepackage{minibox}

\usepackage{tikz}
\usetikzlibrary{calc}
\usetikzlibrary{positioning}
\usepackage{bbm}
\usepackage{multicol}
\usepackage{subcaption}
\usepackage[margin=1in]{geometry}
\usepackage{authblk}

\usepackage{amssymb,amsmath,amsthm,amsfonts,mathtools}

\usepackage{amsmath,amsfonts,bm}

\def\eqref#1{equation~\ref{#1}}

\def\1{\bm{1}}
\def\0{\bm{0}}

\DeclareMathAlphabet{\mathsfit}{\encodingdefault}{\sfdefault}{m}{sl}
\SetMathAlphabet{\mathsfit}{bold}{\encodingdefault}{\sfdefault}{bx}{n}

\newcommand{\E}{\mathbb{E}}

\newtheoremstyle{dotless}{}{}{\itshape}{}{\bfseries}{}{ }{}

\newtheorem{theorem}{Theorem}[section]

\newtheorem{claim}[theorem]{Claim}
\newtheorem{definition}[theorem]{Definition}
\theoremstyle{dotless}

\newcommand{\Ex}[2]{%
\ifthenelse{\isempty{#2}}{\mathop{\E}_{\substack{#1}}}{\mathop{\E}_{\substack{#1\\#2}}}}

\newcommand{\moe}{\text{MOE}}
\newcommand*\circled[1]{\tikz[baseline=(char.base)]{
            \node[shape=circle,draw,inner sep=2pt] (char) {#1};}}
\title{Rip van Winkle's Razor: A Simple Estimate of Overfit to Test Data}
\author[1,2]{Sanjeev Arora}
\author[1]{Yi Zhang}
\affil[1]{Department of Computer Science, Princeton University}
\affil[2]{Institute of Advanced Study}
\date{}
\begin{document}
\maketitle

\begin{abstract}
Traditional statistics forbids use of test data (a.k.a. holdout data) during training. Dwork et al. 2015 pointed out that current practices in machine learning, whereby researchers build upon each other's models, copying hyperparameters and even computer code---amounts to implicitly training on the test set. Thus error rate on test data may not reflect the true population error. This observation  initiated {\em adaptive data analysis}, which provides evaluation mechanisms with guaranteed upper bounds on this difference. With statistical query (i.e. test accuracy) feedbacks, the best upper bound is fairly pessimistic: the deviation can hit a practically vacuous value if the number of models tested is quadratic in the size of the test set.  

In this work, we present a simple new estimate, {\em Rip van Winkle's Razor}. It relies upon a new notion of \textquotedblleft information content\textquotedblright\ of a model: the amount of information that would have to be provided to an expert referee who is intimately familiar with the field and relevant science/math, and who has been just been woken up after falling asleep at the moment of the creation of the test data (like \textquotedblleft Rip van Winkle\textquotedblright\ of the famous fairy tale). This notion of information content is used to provide an estimate of the above deviation which is shown to be non-vacuous in many modern settings.  
\end{abstract}

\section{Introduction}

\label{sec:intro}

 \textquotedblleft Ye shall not train on the test set!\textquotedblright\ is a basic tenet of
data hygiene. It reminds us to  break available data into three parts: {\em training}, {\em validation} (to tune model parameters), and {\em holdout} or {\em test} (for evaluating final performance).
But as highlighted in well-known discussions about \textquote{p-hacking}~\cite{gerring2006case}, \textquote{publication bias}~\cite{mlinaric2017dealing}, \textquote{garden of forking paths}~\cite{gelman2013garden} etc., in practice experimenters use the data to test many models/hypotheses, some of which they came up with after testing earlier hypotheses. 
Dwork et al.~\cite{dwork2015preserving} argued that in modern machine learning (as well as other fields like Genomics), a similar phenomenon --- designing models using the test set--- may be happening. 

They were refering to the fact that machine learning research  is driven by a small number of publicly available datasets.  For example the famous ImageNet Large-Scale Visual Recognition Challenge (ILSVRC) dataset~\cite{russakovsky2015imagenet} consists of training/validation/test sets of size $1,200,000 / 50,000 / 100,000$ (image, class) pairs drawn from $1000$ classes. This dataset has been used in tens of thousands of research papers, and yet has not been refreshed since 2012. For purposes of this paper, error of a model on a data set is simply the fraction of data points on which the model provides the correct answer. Typically the net has more trainable parameters than the number of training data points and will attain error essentially $100\%$  on the training set.  The main metric of interest is the model's error on the full distribution of unseen images, which we refer to as {\em population error}. The difference between error on the training set and the population error is called the {\em generalization error}. In practice the population error is estimated by evaluating the model on images in the test set; we will refer to this as the {\em test error}.\footnote{On ImageNet, most papers developed and tested models on the validation set. In the context to this paper, we use the terms “validation set” and “test set” interchangeably for ImageNet.}

Tens of thousands of teams may work on the same dataset. Teams use automated search to try out a huge number of models and architectures (hyperparameters) and then only publish the best. Furthermore, teams borrow design ideas and even implementation codes from published works. Each published model carries over some information about the test set,  which is inherited in newer models. Over time the amount of this carried-over information can be nontrivial given the vast size of today's models and vast number of papers. This was pointed out in~\cite{dwork2015preserving}, which initiated {\em adaptive data analysis} to study the effects of this phenomenon on statistical estimation of model accuracy.

 To quantitatively understand the issue, let {\em meta-overfitting error} ($\moe$) denote the  difference between average error of a model on the test data points and the expected error on the full distribution. It is well-known that for a model trained without ever querying the test set, this scales as $1/\sqrt{N}$ where $N$ is the size of the test set. Now imagine an experimenter designs $t$ models without ever querying the test set, and then evaluates them in one go on the test set.  Standard concentration bounds imply that the maximum $\moe$ of the $N$ models scales as $O(\sqrt{\log(t)/ N})$. Asymptotically speaking, this error is benign even if the number of models $t$ is   (a small)  exponential in $N$. In machine learning  $N$ is typically at least ten thousand, so  there's little cause for concern. However, in adaptive data analysis the experimenter is allowed to run the first $i-1$ models on the test set before designing the $i$'th model. Here~\cite{dwork2015preserving} shows that $\moe$ of the $t$'th model can be as high as $\Omega(\sqrt{t/N})$, which raises the possibility that the test set could be essentially useless once $t > N$ ---which does hold given the total number of models being trained world-wide on popular datasets. Subsequent work in adaptive data analysis has somewhat increased this estimate of how large $t$ can be~\cite{dwork2015preserving, bassily2016algorithmic}, but the basic point is unchanged.

Thus a lingering suspicion arose that today's deep learning techniques may have \textquotedblleft overfitted to the test set\textquotedblright\ for standard vision datasets such as ImageNet and CIFAR10 and even Kaggle competitions. Recently researchers tried to subject this to empirical study~\cite{recht2018cifar,recht2019imagenet}. They created new test sets for ImageNet and CIFAR10 by carefully replicating the methodology used for constructing the original datasets. Testing famous published models of the past seven years on the new test sets, they found that the test error was higher by as much as $10$-$15\%$ compared to that on the original test set. On the face of it, this seemed to confirm a case of bad $\moe$ but the authors presented some evidence that the swing in test error was due to systemic effects during test set creating. For instance, a comparable swing happens also for models that predated the creation of ImageNet (and thus were not overfitted to the ImageNet test set). A followup study~\cite{roelofs2019meta} of a hundred Kaggle competitions used fresh, identically distributed test sets that were available from the official competition organizers. This study shows little evidence of substantial $\moe$ in Kaggle. 

While the work of~\cite{recht2018cifar,recht2019imagenet} is reassuring at some level, at another level it reminds us that creating a new test set is rife with systematic bias, and is even impossible in many settings where one is studying rare or one-time phenomena (e.g., stock prices). In such cases it is impossible to use a fresh test set to reassure oneself about low $\moe$, so the Dwork et al.\ program to understand $\moe$ remains important. 

\paragraph{This work.} We provide a new and simple upper bound on the $\moe$ that is much less pessimistic than prior estimates, and yields non-vacuous estimates for real-life values of $t, N$. The estimate is based upon a new consideration of \textquotedblleft information content\textquotedblright\ of a model which we call  {Rip van Winkle's Razor}. We work out a couple of examples to show how to apply it in concrete real-life settings. This kind of calculation is well within the expertise of most researchers and we hope it becomes a standard in the field.  

\subsection{Related Works}
\label{subsec:related}

After \cite{dwork2015preserving} identifed the test set overuse issue, the adaptive data analysis community has made attempts towards resolving the issue by designing new restricted ways of using test sets to rank models. Multiple mechanisms~\cite{dwork2015preserving, bassily2016algorithmic} have been shown to yield better upper bounds on $\moe$ scaling as $O(t^{\frac{1}{4}} / \sqrt{N})$. \cite{hardt2014preventing} proves any mechanism with better guarantee would be computationally inefficient to implement. \cite{blum2015ladder} managed to bypass the lower bound by revealing the test accuracy of only \emph{selected} models, and their proof is based an idea seemingly close to our notion of description length though drastically different in nature. These mechanisms are based on ideas borrowed from differential privacy~\cite{dwork2011differential}. However, these works do not attempt to explain the small $\moe$'s observed in real-world settings, where their mechanisms are not implemented. 

Researchers have also tried to give other explanations for why $\moe$ may be small in practice. \cite{zrnic2019natural} suggest that model creators may short memory ---specifically, new models may not tend to use much information from papers published more than a year or two ago ---and use this hypothesis to improve estimates of $\moe$. While  an interesting hypothesis about the sociology of the researchers, it seems ultimately untestable, and furthermore does not give guidance about how to avoid excessive $\moe$ or how to spot it when it happens. \cite{mania2019model} give a better calculation assuming a difficult-to-verify assumption that the mistakes made by models on hard examples are \emph{independent}. \cite{feldman2019advantages} demonstrates the advantage of multiclass classification for reducing $\moe$ on a certain unnatural distribution, but implication for real-world machine learning datasets such as the ImageNet Challenge remains unclear.

\section{Setup and informal version of result}
\label{sec:setup}

We introduce our results informally in this section, and postpone exact calculation  to next section. 
Our technique rests on estimating the information content of a model, a version of the familiar Occam’s razor. A rough idea is encapsulated in the following theorem (essentially folklore) connecting $\moe$ to description length of models. The proof is a standard exercise in concentration bounds.

\begin{theorem}[Informal Theorem, Folklore] \label{thm:informal}
If a model can be  described using $k$ bits, then with high probability over the choice of a test set of size $N$, its meta-overfitting error is at most $c\sqrt{k/N}$ where $c$ is some fixed small constant. 
\end{theorem}
\begin{proof}(sketch) Follows from standard concentration bounds, and trivial union bound over all $2^k$ models that can be described using $k$ bits.\end{proof}

Though folklore, this upper bound on $\moe$ is usually believed to be vacuous (i.e., bigger than $1$) in machine learning. The reason is that the obvious choices for description length $k$ --e.g., number of model parameters, or length of the computer code used to produce it---give vacuous estimates because of the sheer size of today's deep models.  Below we give a non-vacuous formalization of description length. Then in Section~\ref{sec:result} we show that the description length of today's popular models is modest. 

\subsection{Description Length: Rip van Winkle's Razor} 
Our notion of description length exploits the fact that in order to be accepted by the scientific community, models have to be {\em reproducible} by journal (or conference) referees. Specifically, authors have to describe the training method in sufficient detail to allow the referees to reproduce the result using the (universally available) training and test datasets. It is especially relevant that while the number of model parameters is very large,  they are initialized using random numbers at start of training. Thus  referees will use their own random initialization and end up with very different parameters than the research team's model, despite having similar performance on the test set.  The training method is considered reproducible only if it works for most choices of the random initialization (say, at least $50\%$ of the choices). 
 
But a natural issue arises when estimating \textquotedblleft description length\textquotedblright :  it depends upon the referee. A referee ignorant of even basic  calculus might need a very long explanation; an expert referee with up to date knowledge  needs only a tiny one.  What referee can we assume? The answer is subtly different from the usual notion of a journal referee.
 
 \begin{definition}[Informed and Unbiased Referee] An {\em Informed} referee knows {\em everything} that was known to humanity (e.g., about deep learning, mathematics, optimization, statistics etc.) right up to the moment of creation of the held out set. An {\em Unbiased} referee is one who knows nothing discovered since that moment. 
 \end{definition}
 
 \begin{definition}
 The {\em description length} of a deep model  is the minimum number of bits in a description needed to allow an informed but unbiased  referee to reproduce the claimed result.  
 \end{definition}

 Notice, description length for an informed and unbiased referee can be used as $k$ in Theorem~\ref{thm:informal}. The reason is that  being unbiased the referee does not possess any information about the test set, and thus the test set could even be sampled after the referee receives the description.

\noindent{\bf Remarks:}
 \begin{enumerate}
\item	
 Requiring referees to be informed allows descriptions to be shorter. Requiring referees to be  unbiased requires longer descriptions but help rule out statistical contamination due to any interaction whatsoever with the test set. Informally, we can think of the referee as a “Rip van Winkle” figure: an infinitely well-informed researcher who went into deep sleep at the moment of creation of the test set, and has just been woken up to start reproducing the latest result. This is why we call our method the {\em Rip van Winkle's Razor.} Real life journal referees would simulate the idealized Rip van Winkle in their heads while perusing the description submitted by the researcher. 

 \item	To illustrate: momentum techniques in optimization were well-studied before the creation of ImageNet test set, so  Informed referees can be expected to understand a line like “Train with momentum 0.9.”  But a line like \textquoteleft Use Batch Normalization\textquotedblright\ cannot be understood by an  Unbiased referee since conceivably this technique (invented in \cite{ioffe2015batch}) might have become popular precisely because it leads to better performance on test sets of popular machine learning tasks. Thus a referee who knows about it is, for purposes of the Main theorem, not independent of the test set. 


 \item	Implementation details concerning the latest computer hardware need not be included in the description. The referee could use the description to produce a computer program for any computational hardware. Here we are assuming that hardware details only affect the training speed, and not the accuracy of the trained model. 

 \item	To shorten the description as much as possible (to reduce our estimate of $k$ for the main theorem) researchers are allowed to compress it non-destructively using any method that would make sense to Rip van Winkle, and the description of the compression method itself is not counted towards the description length -- provided the same method is used for all papers submitted to Rip van Winkle.  
To give an example, a technique appearing in a text known to Rip van Winkle could  be succinctly refered to using the book ISBN number and page number.
 \end{enumerate}
\section{Main Results}

\subsection{Notations}
We consider the standard formalization of supervised learning where we are presented with a test set of $S = \{(x_1, y_1), \dots (x_N, y_N)\}$ of $N$ data points sampled i.i.d. from a data distribution $\mathcal{D}$ on $\mathcal{X}\times\mathcal{Y}$. Let  $f : \mathcal{X} \rightarrow \mathcal{Y}$ be a classifier assigning a label from $\mathcal{Y}$ to every point in $\mathcal{X}$. We denote the population error of $f$ by $L_{\mathcal{D}}(f)=\mathbb{P}_{(x,y)\sim \mathcal{D}}\left[f(x)\neq y\right]$ and its test error on $S$ by $L_{S}(f)=\frac{1}{N}\sum_{i=1}^N \textbf{1}\{f(x_i)\neq y_i\}$. Now we formally define the aforementioned \emph{meta-overfitting error} as the difference between $L_{\mathcal{D}}(f)$ and $L_{S}(f)$.

Let $\mathcal{H}$ denote a set of classifiers on the domain $\mathcal{X}\times\mathcal{Y}$, and let $\{0,1\}^{\leq C}$ denote the set of all binary strings with at most $C$ bits.  We represent the \emph{informed} and \emph{unbiased} referee as a function $r:\{0,1\}^{\leq C}\rightarrow \mathcal{H}\cup\{\emptyset\}$. That is, given a binary description, the referee either outputs a classifier from $\mathcal{H}$ or $\emptyset$ when, for example, the referee does not understand the description. Crucially, the referee function $r$ is independent of the test set $S$ (unbiasedness). For a description $\sigma \in \{0,1\}^{\leq C}$, we denote its length by $|\sigma|$, and for a classifier $f$, we denote its minimum description length by $|f|_r:=\min\left\{ |\sigma| ~\middle| ~\sigma\in \{0,1\}^{\leq C} \ \text{and}\ r(\sigma)=f \right\}$. 

\subsection{Main Theorem}

\begin{theorem}
Let $r$ be the informed but unbiased refree and let $\mathcal{H}:=\left\{r(\sigma) \middle| \sigma\in\{0,1\}^{\leq C}\right\}$. Then for every confidence parameter $\delta>0$, with probability at least $1-\delta$ over the choice of the test set $S\sim \mathcal{D}^N$, we have that
\begin{align*}
\forall f\in\mathcal{H}, L_{D}(f) \leq L_S(f) + \sqrt{\frac{2\ln{2}\ p_f^* (1-p_f^*) \left( |f|_r + \log_2\left(C/\delta\right) \right)}{N}}
\end{align*}
where $\hat{p}_f:=L_{S}(f)$ and $p_f^*$ is the fixed point of the function $T(p) := \hat{p}_f + \sqrt{\frac{2\ln{2}\ p (1-p)\ \left( |f|_r + \log_2\left(C/\delta\right) \right)}{N}}$, i.e. $p_f^*=T(p_f^*)$.
\end{theorem}

\begin{proof}
We first prove a generalization result relying on the knowledge of the population error $p_f:=L_D(f)$, then we show how to adapt it to depend on only the observed test error $\hat{p}_f:=L_S(f)$.

To prove the first result, we decompose $\mathcal{H}$ into disjoint subsets: $\mathcal{H} = \cup_{i=1}^C \mathcal{H}_s $ where $\mathcal{H}_s:=\{f\in\mathcal{H}: |f|_r = s\}$. For a fixed $s$, we use the lower tail Chernoff bound~\ref{claim:chernoff} and union bound over all classifiers in $\mathcal{H}_s$ and conclude that
\begin{align*}
\forall f\in\mathcal{H}_s, \mathbb{P}\left[\frac{1}{N}\sum_{i=1}^N \textbf{1}\{f(x_i)\neq y_i\}\leq p_f - \varepsilon \right] \leq 2^s \exp\left(-\frac{N\varepsilon^2}{2p_f(1-p_f)}\right)
\end{align*}
It follows that for a fixed $s$, we have, with probability $\geq 1- \delta/C$,
\begin{align*}
\forall f\in\mathcal{H}_s, L_{D}(f) \leq L_S(f) + \sqrt{\frac{2\ln{2}\ p_f(1-p_f) \left( s + \log_2\left(C/\delta\right) \right)}{N}}
\end{align*}
We complete the first part with a union bound over all integer values of $s$ between $1$ and $C$.

In the following, we show how to make the generalization bound depend on $\hat{p}_f$ instead of $p_f$. The intuition is that we can always start with a trivial estimate, for instance, $p_f\leq 0.5$. The first part shows that, with probability $1-\delta$, we have $p_f\leq T(0.5)$. As long as $T(0.5)\leq0.5$, we have obtained a better estimates. Then we can repeat the process to get better and better estimates. This means in order to obtain the tightest upper bound on $p_f$, we seek to find the smallest $p$ such that $T(p)\leq p$. The equation $T(p)= p$ can be transformed into a quadratic equation which always has two distinct real roots $0<p_1<p_2$, but only one of the roots satisfy the sanity check $p_2\geq\hat{p}_f$. Thus $p_2$ is the unique solution to the equation $T(p)=p$, and $p_2=p_f^*$. Furthermore, it is easy to verify that $T(p) \leq p$ for all $p\geq p_f^*$. Thus $p_f^*$ is indeed the smallest $p$ such that $T(p)\leq p$. 

\end{proof}

\subsection{Meta-Overfitting Errors of ImageNet Models}
\label{sec:result}
We provide Rip van Winkle with the descriptions for reproducing two mainstream ImageNet models, ResNet-152~\cite{he2016deep} and DenseNet-264~\cite{huang2017densely}, in Appendix~\ref{sec:descriptions}, and then discuss in detail the strategy of estimating their lengths in Section~\ref{sec:encode}. Here we report the estimated description lengths of the models as well as the upper bounds on their population errors implied by our main theorem.

\paragraph{AlexNet as a baseline.} AlexNet's wide margin in the ImageNet LSVRC-2012 contest~\cite{krizhevsky2012imagenet} has been regarded as the starting point of our current deep learning revolution, and has significantly influenced the designs of later models. However, the development of AlexNet might in fact be independent of the ImageNet LSVRC-2012 contest dataset. The original paper with all the details of model architecture, data augmentation, training and testing procedures was published at NeurIPS 2012 whose paper submission deadline was on June 1st {\emph prior} to the starting date of the 2012 contest. Furthermore, the authors acknowledged in the paper that the model was developed on the ImageNet LSVRC-2010 dataset, and simply \textquotedblleft also entered a variant of this model in the ILSVRC-2012 competition\textquotedblright. If the procedures mentioned in their paper are indeed independent of the 2012 dataset, it will significantly reduce the description length, since ResNet and DenseNet inherited data augmentation and training/testing protocols from AlexNet. We present our results in Table~\ref{tab:result1} and ~\ref{tab:result2} with and without AlexNet as a baseline. These results are based on the parameter choice of $C = 5000, \delta = 0.05, N = 50000$.

\begin{table}[H]
\begin{center}
\begin{tabular}{|c|c|c|c|c|c|}
\hline
\multirow{2}{*}{Model} & \multirow{2}{*}{top-5 val error} & \multicolumn{2}{|c|}{w/ AlexNet as baseline} & \multicolumn{2}{|c|}{w/o  AlexNet as baseline}  \\\cline{3-6}
             &                                            &   desc. length &  our bound             & length (bit) & our bound\\\hline
ResNet-152   & 4.49        \%                             &  426 bits        &  7.39       \%                 &   729 bits      &      8.49\% \\\hline
DenseNet-264 & 5.29        \%                             &  362 bits         & 8.08 \%                            &        741 bits      &      9.55\% \\\hline
\end{tabular}
\caption{Description lengths and our upper bounds on {\em population error} with English counted as $1.0$ bits per character (Option 1 in Section~\ref{sec:english}).}
\label{tab:result1}
\end{center}
\end{table}\vspace{-1cm}
\begin{table}[H]
\begin{center}
\begin{tabular}{|c|c|c|c|c|c|}
\hline
\multirow{2}{*}{Model} & \multirow{2}{*}{top-5 val error} & \multicolumn{2}{|c|}{w/ AlexNet as baseline} & \multicolumn{2}{|c|}{w/o  AlexNet as baseline}  \\\cline{3-6}
             &                                            &   desc. length &  our bound             & length (bit) & our bound\\\hline
ResNet-152   & 4.49        \%                             &  556 bits        &  7.89       \%                 &   1032 bits      &      9.49\% \\\hline
DenseNet-264 & 5.29        \%                             &  454 bits         & 8.47 \%                            &        980 bits      &      10.35\% \\\hline
\end{tabular}
\caption{Description lengths and our upper bounds on {\em population error} with English counted as $10$ bits per word (Option 2 in Section~\ref{sec:english}).}
\label{tab:result2}
\end{center}
\end{table}%

\section{Counting Bits in Descriptions}
\label{sec:encode}
The descriptions we provide to Rip van Winkle consist of three types of expressions: English, mathematical equations and directed graphs (for network architectures). While the use of these expressions appears natural, systematically counting the number of bits entails clearly defined rubrics. In this section we present the rubrics we use to encode each type of expressions into binary strings and then estimate the description lengths.

\subsection{Description Length of English Expressions}
\label{sec:english}
English is natural for describing certain seemingly sophisticated deep learning practices, for example, data augmentation procedures and testing protocols. Universal binary encodings such as ASCII and Unicode are too wasteful for our purpose, due to their failure to capture the  regularity of natural languages. In this work, we adopt two alternative methods of length measuring at character level and word level respectively.

\paragraph{Option 1) Entropy Rate.} Entropy rate of a string reveals the minimal number of bits needed to encode each character in the string. Estimating the entropy rate of written English has been a long-standing subject in linguistics, dating back to Shannon's seminal work in 1951 where he proposed an estimate between $0.6$ and $1.3$ bits per character~\cite{shannon1951prediction}. More recent studies~\cite{mahoney1999refining, ren2019entropy} have managed to tighten the upper bound to $\sim 1.2$ bits per character. Note that these estimates are on the distribution of general written English, and English used in academic and technical writings usually have a significantly lower entropy rate. Indeed, the entropy rate of a corpus of wikipedia articles, namely the $\mathsf{enwik8}$ dataset, can be upper bounded by merely $0.99$ bits per character\footnote{See~\cite{braverman2020calibration} for how $\log(\text{test perplexity})$ of a trained model upper bounds the entropy rate of the corpus distribution.}, with the state-of-the-art deep learning language models~\cite{dai2019transformer}. We expect the English used in deep learning papers and particularly in our descriptions provided to Rip van Winkle, to exhibit a even lower entropy rate, compared to wikipedia articles. We leave precisely estimating the quantity to future works. For current purposes, we pick $1.0$ bits per character as the entropy rate of the English expressions appearing in our descriptions. 

\paragraph{Option 2) Simplified Technical English.} According its official website~\footnote{http://www.asd-ste100.org/}, \textquotedblleft Simplified Technical English (STE), is a controlled language developed in the early 1980s to help the users of English-language maintenance documentation understand what they read. It was initially applicable to commercial aviation. Then, it became also a requirement for Defence projects, including Land and Sea vehicles. As a consequence, today, primary texts of maintenance manuals are mostly written in STE\textquotedblright. ~Its $2005$ version (ASD-STE100) contains a set of restrictions on the grammar and style of procedural and descriptive text, as well as a dictionary of 875 approved general words. We argue that it is possible to create an equivalent version of Simplified Technical English for deep learning that contains no more than 875 words, based on a plausible belief that training a deep neural network is at least as simple as flying a commercial jetliner, if not simpler. We leave precisely constructing the deep learning counterpart to future work, but provide a partial vocabulary in Section\ref{sec:vocab}. For now, we pick $\lceil \log_2 875 \rceil=10$ bits per world as the description length for each English word in our description.  

\subsection{Description Length of Mathematical Equations}
We estimate the description length of math equations by first converting them into directed computation graphs, and then count the number of bits needed to represent the graph. In the following, we present the details using a concrete example.
\paragraph{Example: batch-normalization.} Batch-normalization is an essential technique ubiquitous in modern architectures, invented after the creation of the ImageNet test set. The mathematical operation within a batch-normalization layer can be precisely defined as the below equation:
\begin{align*}
 \mathsf{BN}(x) := b + g \cdot (x - \mu) / \sqrt{\sigma^2 + 0.01}
\end{align*}
together with specifications that the layer takes a batch as input, and $x$ denotes the activation of a node, $\mu, \sigma^2$ are the batch mean and variance of $x$, and $b, g$ are trainable scalars tied to the node, initialized to $0, 1$. The equation can be readily translated into the directed graph shown in Figure~\ref{fig:bn}.

\begin{figure}[h]
\centering

\begin{tikzpicture}[
rectnode/.style={rectangle, draw=black, thin,  minimum size=5mm},
layernode/.style={rectangle, draw=black, thin,  minimum size=5mm},
circlenode/.style={circle, draw=black, thin, inner sep=0pt,  minimum size=5mm},
align=center,node distance=0.5cm,
]

\node[rectnode, fill=red!9]   (x)           {$x$};
\node[rectnode]   (mean)      [below=of x] {$\mu$};
\node[circlenode] (subtract) [right=1cm of $(x)!0.5!(mean)$]    {$-$};

\node[rectnode]   (sigma2)    [below=of mean]     {$\sigma^2$};
\node[circlenode] (plus1)      [right=of sigma2]   {$+$};
\node[rectnode]   (const)     [below=of plus1]     {$0.01$};  
\node[circlenode] (sqrt)      [right=of plus1]     {\small sqrt};
\node[circlenode] (divide)    [right=1cm of $(subtract)!0.5!(sqrt)$] {$\div$};
\node[circlenode] (mult)      [right=of divide]   {$\times$};
\node[rectnode]   (g)         [below=of mult]     {$g$};
\node[circlenode]  (plus2)     [right=of mult]     {$+$};
\node[rectnode]   (b)         [above=of plus2]     {$b$}; 
\node[rectnode, fill=blue!9]  (out) [right=1cm of plus2] {output};

\draw[->, thick]  (x) -- (subtract) node[pos=0.8, above] {$1$};
\draw[->, thick]  (mean) -- (subtract) node[pos=0.8, below] {$2$};
\draw[->, thick]  (subtract) -- (divide) node[pos=0.8, above] {$1$};
\draw[->, thick]  (sigma2) -- (plus1);
\draw[->, thick]  (const) -- (plus1);
\draw[->, thick]  (plus1) -- (sqrt);
\draw[->, thick]  (sqrt) -- (divide) node[pos=0.8, below] {$2$};
\draw[->, thick]  (divide) -- (mult);
\draw[->, thick]  (mult) -- (plus2);
\draw[->, thick]  (b) -- (plus2);
\draw[->, thick]  (g) -- (mult);
\draw[->, thick]  (plus2) -- (out);
\end{tikzpicture}
\caption{Computation graph of batch-normalization. Variables and math operators are represented by squares and circles. The numbers on the edges suggest the order of arguments. The red and blue indicate the input and output variables.} 
\label{fig:bn}
\end{figure}
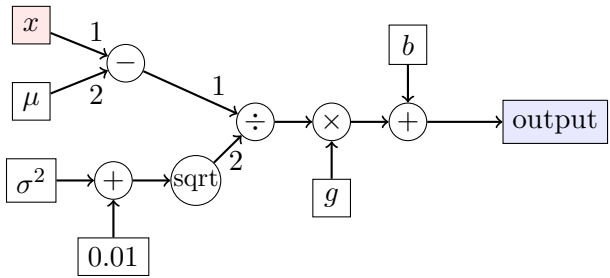

The graph consists of seven variable vertices (square), five operator vertices (circle), and twelve directed edges (arrow). With $7+5=12$ vertices in total, we assign each vertex a $\lceil\log_2 12\rceil=4$-bit index. The input and output vertices are assigned the lowest and highest index number. Because this graph is sparsely connected, we may encode the edges using one list for each vertex containing all the indices of vertices it is connecting into. For order sensitive operators that require two arguments (i.e. division, subtraction), we may use an one bit suffix for each such edge to suggest the order. In this way the edges can be represented with $12\times5=60$ bits.

In addition to edges, we also need to specify the meaning of vertices. For the ones that are simply abstract symbols (i.e. $x$, $g$ and $b$), we leave them as they are. For the remaining vertices, we make a legend where each vertex index is associated to a binary representation. For the vertex with the scalar $0.01$, we use its $8$-bit floating point representation. For the operator vertices, they are among the $25 ~\textbf{Math Operations}$ in Rip van Winkle's vocabulary in Section~\ref{sec:vocab}. Thus we may encode every operator from the list using $\lceil\log_2 25\rceil=5$ bits. Thus the legend contains $8+5\times5=33$ bits. Overall, we can represent the computation graph of batch-normalization using only $\approx100$ bits.

\subsection{Describing Network Architectures}
Deep learning architectures are naturally represented as directed graphs, thus we adopt the same strategy as used to encode math equations. For instance, in Figure~\ref{figure:residual}, we describe the essential ingredient in ResNet architectures\textemdash residual block~\cite{he2016deep}. 

\begin{figure}[h]
\centering
\begin{subfigure}{0.3\textwidth}
\centering
\begin{tikzpicture}[
rectnode/.style={rectangle, draw=black, thin,  minimum size=5mm},
layernode/.style={rectangle, draw=black, thin,  minimum size=5mm},
circlenode/.style={circle, draw=black, thin,  minimum size=1mm},
align=center,node distance=0.3cm,
]
\small
\node[]      (input)           {};
\node[rectnode] (bn1)      [below=of input]      {$\mathsf{BN}$};
\node[rectnode] (relu1)      [below=of bn1]      {$\mathsf{ReLU}$};
\node[rectnode]    (conv1)           [below=of relu1] {$\mathsf{Conv}(1\text{x}1, 4k)$};
\node[rectnode] (bn2)      [below=of conv1]      {$\mathsf{BN}$};
\node[rectnode] (relu2)      [below=of bn2]      {$\mathsf{ReLU}$};
\node[rectnode]    (conv2)           [below=of relu2] {$\mathsf{Conv}(3\text{x}3, k)$};
\node[]             (output)            [below=of conv2]    {};

\draw[->, thick]  (input) -- (bn1);
\draw[->, thick] (bn1) -- (relu1);
\draw[->, thick] (relu1) -- (conv1);
\draw[->, thick] (conv1) -- (bn2);
\draw[->, thick] (bn2) -- (relu2);
\draw[->, thick] (relu2) -- (conv2);
\draw[->, thick] (conv2) -- (output);
\end{tikzpicture}
\caption{}
\end{subfigure}%
\begin{subfigure}{.3\textwidth}
\centering
\begin{tikzpicture}[
downsamplenode/.style={rectangle, draw=black, thin, minimum size=5mm},
layernode/.style={rectangle, draw=black, thin, minimum size=5mm},
addnode/.style={circle, draw=black, thin, minimum size=5mm},
align=center,node distance=0.3cm,
]
\small
\node[]      (input)           {};
\node[downsamplenode] (downsample)      [below=of input]      {$\mathsf{downsample}(s)$};
\node[layernode]    (layer_1)           [below=of downsample] {$\mathsf{Layer}(1\text{x}1, k, s)$};
\node[layernode]    (layer_2)           [below=of layer_1]    {$\mathsf{Layer}(3\text{x}3, k, 1)$};
\node[layernode]    (layer_3)           [below=of layer_2]    {$\mathsf{Layer}(1\text{x}1, 4k, 1)$};
\node[addnode]      (add)               [below=of layer_3]    {$+$};
\node[]             (output)            [below=of add]    {};
\node[] (phantom_1) [right=of downsample] {};
\node[] (phantom_2) at (add -| phantom_1) {};

\draw[->, very thick]  (input) -- (downsample);
\draw[->, very thick] (downsample) -- (layer_1);
\draw[->, very thick] (layer_1) -- (layer_2);
\draw[->, very thick] (layer_2) -- (layer_3);
\draw[->, very thick] (layer_3) -- (add);
\draw[->, very thick] (layer_3) -- (add);
\draw[->, very thick] (add)  -- (output);
\draw[very thick]  (downsample) -- (phantom_1);
\draw[very thick]  (phantom_1.west) -- (phantom_2.west);
\draw[->, very thick]  (phantom_2) -- (add);
\end{tikzpicture}
\caption{}
\end{subfigure}
\caption{a)~$\mathsf{Layer}(k, s)$ and b)~$\mathsf{block}(k,s)$. $k$ is the parameter determining the number of output channels, and $s$ denotes the stride of the down-sampling layer and the first convolutional layer.}
\label{figure:residual}
\end{figure}
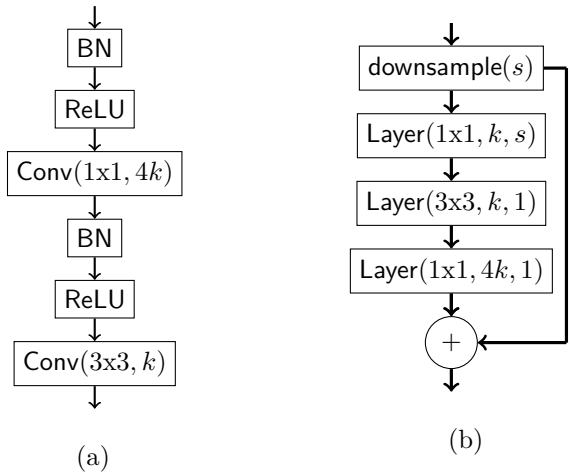

We first define the atomic structure $\mathsf{Layer}$ using $\mathsf{BN}$ defined in the previous section, and primitive layers, namely $\mathsf{ReLU}$ and $\mathsf{Conv}$, from the list of $10~\textbf{Deep Learning Layers}$. Then we build a residual block shown on the right using multiple $\mathsf{Layer}$s, one $\mathsf{downsample}$ layer, and one skip-link. The entire network architecture can be recursively described, and the description length can be calculated in a likewise manner as in the previous section. Rip van Winkle's vocabulary contains $10$ primitive layer types, thus each type may be specified using $\lceil\log_2 10 \rceil = 4$ bits. However, in addition to specifying the type, it takes several extra bits to also encode the hyper-parameters of a layer, including filter size, number of output channels, stride size and so forth.

With the description of a $\mathsf{Block}$, the forward pass of a ResNet-152 now can be readily represented as a chain with $k=32$: $\mathsf{Conv}(7, 64, 2) ~\rightarrow ~\mathsf{MaxPool}(3\times3, 2)~\rightarrow ~\mathsf{Block}(k, 2) \times 3 \rightarrow ~\mathsf{Block}(2k, 2) \times 8 ~\rightarrow ~\mathsf{Block}(4k, 2) \times 36 ~\rightarrow ~\mathsf{Block}(8k, 2)\times 3 \rightarrow ~\mathsf{AvgPool}(\text{global}) ~\rightarrow ~\mathsf{FullyConnected(1000)} ~\rightarrow ~\mathsf{SoftMax}$, where $\times$ denotes replication, i.e. $\mathsf{Block}(4k, 2) \times 36$ represents a string of $36 ~\mathsf{Block}$s with $4k$ output channels and stride $2$. The forward pass chain consists of $5$ primitive layers from Rip van Winkle's \textbf{Deep Learning Layers}, $4$ types of $\mathsf{Block}$s, $16$ hyper-parameters, plus $9$ edges. Thus overall, the forward pass requires $4\times5 + 39 + 3\times9=86$ bits.
\subsection*{Acknowledgement}
Sanjeev Arora acknowledges funding from the NSF, ONR, Simons Foundation, Schmidt Foundation, DARPA and SRC. Yi Zhang acknowledges the support from the Wallace Memorial Fellowship.

\bibliographystyle{plain}
\bibliography{main}

\appendix
\section{Technical Preliminaries}
\begin{claim}
\label{claim:kl_lowerbound}
For any $p\in\left(0, \frac{1}{2}\right]$ and any $\varepsilon\in[0, p)$, we define function $\textrm{KL}\left(p-\varepsilon\middle|| p\right):= (p-\varepsilon)\log\left(\frac{p-\varepsilon}{p}\right) + (1-p+\varepsilon)\log\left(\frac{1-p+\varepsilon}{1-p}\right)$, and then we have
\begin{align*}
\textrm{KL}\left(p-\varepsilon\middle|| p\right)\geq\frac{\epsilon^2}{2p(1-p)}
\end{align*}
\end{claim}
\begin{proof} We write $g(\varepsilon):=\textrm{KL}\left(p-\varepsilon\middle|| p\right)$, and note that $g(0)=g'(0)=0$ and $g''(\varepsilon)=\frac{1}{(p-\varepsilon)(1-p+\varepsilon)}$. Thus for any $p\in\left(0, \frac{1}{2}\right]$ and any $\varepsilon\in[0, p)$, 
\begin{align*}
g(\varepsilon) = \int_0^{\varepsilon} \int_0^{u} g''(\tau) \mathrm{d}\tau \mathrm{d}u = \int_0^{\varepsilon} \int_0^{u} \frac{1}{(p-\tau)(1-p+\tau)} \mathrm{d}\tau \mathrm{d}u \geq \int_0^{\varepsilon} \int_0^{u} \frac{1}{p(1-p)} \mathrm{d}\tau \mathrm{d}u = \frac{\epsilon^2}{2p(1-p)}
\end{align*}

\end{proof}

\begin{claim}[Chernoff Bound, lower tail]
\label{claim:chernoff}
Suppose $X_1, X_2, \ldots, X_N$ are independen Bernoulli variables with $\E[X_i]=p$ for all $i$. Then for any $\varepsilon\in[0,p)$,
\begin{align*}
\mathbf{P}\left[ \frac{1}{N}\sum_{i=1}^N X_i \leq p - \varepsilon\right]\leq \exp\left(-\frac{N\varepsilon^2}{2p(1-p)}\right)
\end{align*}
\end{claim}
\begin{proof}
We introduce another $N$ Bernoulli variables $Z_i := 1-X_i$, and $Z:=\sum_{i=1}^n Z_i$ so that
\begin{align*}
\mathbf{P}\left[\frac{1}{N}\sum_{i=1}^N X_i \leq p - \varepsilon\right] = \mathbf{P}\left[\frac{1}{N}\sum_{i=1}^N Z_i \geq 1-p + \varepsilon\right] = \mathbf{P}\left[Z \geq N(1-p + \varepsilon)\right]
\end{align*}
For any $\lambda>0$, we have
\begin{align*}
\mathbf{P}\left[Z \geq N(1-p + \varepsilon)\right] = \mathbf{P}\left[e^{\lambda Z} \geq e^{\lambda N(1-p + \varepsilon)}\right] \leq \frac{\E\left[e^{\lambda Z}\right]}{e^{\lambda N(1-p + \varepsilon)}}.
\end{align*}
where we used Markov's inequality. The independence of $Z_i$'s yields
\begin{align*}
\E\left[e^{\lambda Z}\right] = \E\left[\prod_{i=1}^N e^{\lambda Z_i}\right] = \prod_{i=1}^N \E\left[e^{\lambda Z_i}\right] = \left((1-p)e^\lambda + p\right)^N
\end{align*}
Thus
\begin{align*}
\mathbf{P}\left[Z \geq N(1-p + \varepsilon)\right] \leq \left(\frac{(1-p)e^\lambda + p}{e^{\lambda (1-p + \varepsilon)}}\right)^N
\end{align*}
Minimizing the right hand side over $\lambda > 0$ we obtain
\begin{align*}
\mathbf{P}\left[ \frac{1}{N}\sum_{i=1}^N X_i \leq p - \varepsilon\right] \leq \left[\left(\frac{p}{p-\varepsilon}\right)^{p-\varepsilon} \left(\frac{1-p}{1-p+\varepsilon}\right)^{1-p+\varepsilon}\right]^N = e^{-N \cdot \textrm{KL}\left(p-\varepsilon\middle|| p\right)}
\end{align*}
Invoking claim~\ref{claim:kl_lowerbound} completes the proof.
\end{proof}
\section{Rip van Winkle's Vocabulary}
\label{sec:vocab}
Rip van Winkle's vocabulary consists of primitive operations/functions independent of the ImageNet test set. \textbf{Math Operations} are selected selected from Numpy's Mathematical
and Linear algebra routines~\cite{harris2020array}. All the other math operations in Numpy can be constructed from our selected ones. For instance, $\tan(x):=\mathsf{divide}(\mathsf{sin}(x), \mathsf{cos}(x))$, $\sinh(x):=\mathsf{divide}(\mathsf{subtract}(\mathsf{exp}(x), \mathsf{exp}(\mathsf{subtract}(0, x))), 2)$ and $\pi:=\mathsf{multiply}(2, \mathsf{asin}(1))$. The \textbf{Random Sampling Functions} are mostly used for describing network parameter initializations and data augmentation procedures. Notably, we include $\mathsf{SetRNGSeed}$ which specifies the seed of the Random Number Generator. \textbf{Neural Network Layers} and \textbf{Optimizers} contain typical deep neural network layer types and optimizers invented {\emph before} Rip Van Winkle was put into deep sleep at the year of $2012$\footnote{AdaDelta was published in December 2012, after the ImageNet 2012 competition. However, the original paper did not experiment on ImageNet}. Note that we are always allowed to introduce novel layers and optimizers using plain English and the operations listed here, as long as we account for their description lengths as well.

\begin{description}
\item[Math Operations.] \hfill \\
$\mathsf{add}, \mathsf{subtract},\mathsf{multiply}, \mathsf{divide}, \mathsf{mod}, \mathsf{sin}, \mathsf{arcsin},  \mathsf{exp}, \mathsf{log}, \mathsf{power}, \mathsf{round}, \mathsf{clip}, \mathsf{sqrt}, $\\
$\mathsf{abs}, \mathsf{sign}, \mathsf{max}, \mathsf{argmax}, \mathsf{dot}, \mathsf{matmul}, \mathsf{svd}, \mathsf{pseudo\mbox{-}inverse}, \mathsf{kronecker\mbox{-}product}, \mathrm{i}, \mathsf{Re}, \mathsf{Im}$.

\item[Random Sampling Functions.] \hfill\\
$\mathsf{N}(\mu, \Sigma)$, $\mathsf{Laplace}$, $\mathsf{Uniform}$, $\mathsf{Bernoulli}$, $\mathsf{Beta}$, $\mathsf{Multinomial}$, $\mathsf{Poisson}$, $\mathsf{RandInt}$, $\mathsf{SetRNGSeed}$.

\item[Tensor Operations:] \hfill\\
$\mathsf{index}$, $\mathsf{concat}$, $\mathsf{split}$, $\mathsf{reshape}$, $\mathsf{copy}$

\item[Neural Net Layers:] \hfill\\
$\mathsf{Conv}$(filter size, num output channels, stride), $\mathsf{FullyConnected}$(num output channels),  $\mathsf{ReLU}$, $\mathsf{Sigmoid}$, $\mathsf{Threshold}$, $\mathsf{SoftMax}$, $\mathsf{MaxPooling}$(filter size, stride), $\mathsf{AvgPooling}$(filter size, stride), $\mathsf{Downsample}$(stride), $\mathsf{Dropout}(p)$

\item[Optimizers:] \hfill\\
$\mathsf{SGD}$, $\mathsf{GradientDescent}$, $\mathsf{AdaGrad}$, $\mathsf{AdaDelta}$, $\mathsf{RMSProp}$
\end{description}
\section{Descriptions of ImageNet Models}
\label{sec:descriptions}
\subsection{Description of ResNet-152 on ImageNet ($4.49$\% top-5 error)}

{
\small 

\begin{description}[align=left]
\item[Batch-Normalization:] \hfill \\
  at each neuron $x$ apply 
  \begin{flalign*}
  &\mathsf{BN}(x) = b + g \cdot (x - \mu) / \sqrt{\sigma^2 + 0.01} &
  \end{flalign*}
  $\mu, \sigma^2:$ batch mean, variance of $x$ (test time use train set stats)\\
  $b, g:$ trainable, init $b = 0, g = 1$

\item[Architecture:] \hfill \\ 
$\mathsf{Layer}(k, s)$:

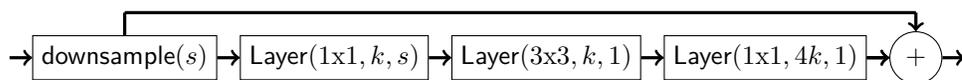
\begin{tikzpicture}[
rectnode/.style={rectangle, draw=black, thin,  minimum size=5mm},
layernode/.style={rectangle, draw=black, thin,  minimum size=5mm},
circlenode/.style={circle, draw=black, thin,  minimum size=1mm},
align=center,node distance=0.3cm,
]
\small
\node[]      (input)           {};
\node[rectnode] (bn1)      [right=of input]      {$\mathsf{BN}$};
\node[rectnode] (relu1)      [right=of bn1]      {$\mathsf{ReLU}$};
\node[rectnode]    (conv1)           [right=of relu1] {$\mathsf{Conv}(1\text{x}1, 4k, s)$};
\node[rectnode] (bn2)      [right=of conv1]      {$\mathsf{BN}$};
\node[rectnode] (relu2)      [right=of bn2]      {$\mathsf{ReLU}$};
\node[rectnode]    (conv2)           [right=of relu2] {$\mathsf{Conv}(3\text{x}3, k, s)$};
\node[]             (output)            [right=of conv2]    {};

\draw[->, thick]  (input) -- (bn1);
\draw[->, thick] (bn1) -- (relu1);
\draw[->, thick] (relu1) -- (conv1);
\draw[->, thick] (conv1) -- (bn2);
\draw[->, thick] (bn2) -- (relu2);
\draw[->, thick] (relu2) -- (conv2);
\draw[->, thick] (conv2) -- (output);
\end{tikzpicture}

$\mathsf{Block}(k, s)$:

\begin{tikzpicture}[
downsamplenode/.style={rectangle, draw=black, thin, minimum size=5mm},
layernode/.style={rectangle, draw=black, thin, minimum size=5mm},
addnode/.style={circle, draw=black, thin, minimum size=5mm},
align=center,node distance=0.3cm,
]
\small
\node[]      (input)           {};
\node[downsamplenode] (downsample)      [right=of input]      {$\mathsf{downsample}(s)$};
\node[layernode]    (layer_1)           [right=of downsample] {$\mathsf{Layer}(1\text{x}1, k, s)$};
\node[layernode]    (layer_2)           [right=of layer_1]    {$\mathsf{Layer}(3\text{x}3, k, 1)$};
\node[layernode]    (layer_3)           [right=of layer_2]    {$\mathsf{Layer}(1\text{x}1, 4k, 1)$};
\node[addnode]      (add)               [right=of layer_3]    {$+$};
\node[]             (output)            [right=of add]    {};
\node[] (phantom_1) [above=of downsample] {};
\node[] (phantom_2) at (add |- phantom_1) {};

\draw[->, very thick]  (input) -- (downsample);
\draw[->, very thick] (downsample) -- (layer_1);
\draw[->, very thick] (layer_1) -- (layer_2);
\draw[->, very thick] (layer_2) -- (layer_3);
\draw[->, very thick] (layer_3) -- (add);
\draw[->, very thick] (layer_3) -- (add);
\draw[->, very thick] (add)  -- (output);
\draw[very thick]  (downsample) -- (phantom_1);
\draw[very thick]  (phantom_1.south) -- (phantom_2.south);
\draw[->, very thick]  (phantom_2) -- (add);
\end{tikzpicture}




\item[Forward-Pass:] \hfill\\
$k=64$\\
$\mathsf{Conv}(7, 64, 2) ~\rightarrow ~\mathsf{MaxPool}(3\times3, 2)~\rightarrow ~\mathsf{Block}(k, 2) \times 3 \rightarrow ~\mathsf{Block}(2k) \times 8 ~\rightarrow$\\
$ ~\mathsf{Block}(4k, 2) \times 36 ~\rightarrow ~\mathsf{Block}(8k, 2)\times 3 \rightarrow ~\mathsf{AvgPool}(global) ~\rightarrow ~\mathsf{FullyConnected(1000)} ~\rightarrow ~\mathsf{SoftMax}$\\

\item[Initialization:] ~~~ $N(0, 2 / \text{fan-in})$, bias $0$\\
\item[Data-Augmentation] \hfill \\ 
scale pixels to $0$ mean  unit variance\\
resize its shorter side to $256$\\
random crop $224\times224$ with horizontal flip \\
SVD $3\times 3$ covariance matrix of RBG pixels over training set: $\lambda_i, \mathbf{v}_i$\\
sample $\alpha_i\sim N(0, 0.01)$ for each image, add $\sum_i \alpha_i \lambda_i \mathbf{v}_i$ to pixels, re-sample every epoch
\item[Training] \hfill\\
SGD(batchsize=256, weight-decay=1e-4, momentum=0.9, iteration=60e4) \\
learningrate init 0.1,  learningrate $/=10$ every $30$ epochs\\
\item[Testing] \hfill\\
	full convolution at 224 256 384 480 640 with horizontal flips, average logits
\end{description}  
\subsection{Description of DenseNet-264 on ImageNet ($5.29\%$ top-5 error)}

{
\small 

\begin{description}[align=left]
\item[Batch-Normalization:] \hfill \\
  at each neuron $x$ apply 
  \begin{flalign*}
  &\mathsf{BN}(x) = b + g \cdot (x - \mu) / \sqrt{\sigma^2 + 0.01} &
  \end{flalign*}
  $\mu, \sigma^2:$ batch mean, variance of $x$ (test time use train set stats)\\
  $b, g:$ trainable, init $b = 0, g = 1$

\item[Architecture:] \hfill \\ 
$\mathsf{Layer}(k):$\\

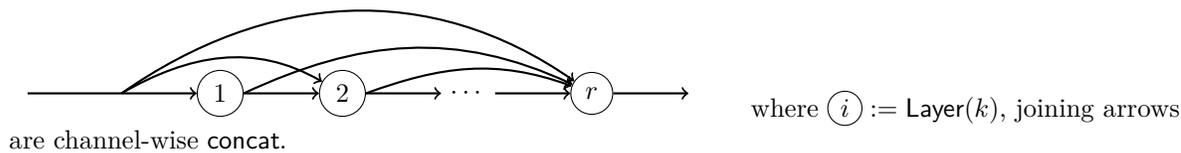
\begin{tikzpicture}[
rectnode/.style={rectangle, draw=black, thin,  minimum size=5mm},
layernode/.style={rectangle, draw=black, thin,  minimum size=5mm},
circlenode/.style={circle, draw=black, thin,  minimum size=1mm},
align=center,node distance=0.5cm,
]
\node[]      (input)           {};
\node[rectnode] (BNReLU1)      [right=of input]      {$\mathsf{BN}+\mathsf{ReLU}$};
\node[rectnode]    (conv1)           [right=of BNReLU1] {$\mathsf{Conv}(1\text{x}1, 4k)$};
\node[rectnode] (BNReLU2)      [right=of conv1]      {$\mathsf{BN}+\mathsf{ReLU}$};
\node[rectnode]    (conv2)           [right=of BNReLU2] {$\mathsf{Conv}(3\text{x}3, k)$};
\node[]             (output)            [right=of conv2]    {};

\draw[->, thick]  (input.east) -- (BNReLU1.west);
\draw[->, thick] (BNReLU1.east) -- (conv1.west);
\draw[->, thick] (conv1.east) -- (BNReLU2.west);
\draw[->, thick] (BNReLU2.east) -- (conv2.west);
\draw[->, thick] (conv2.east) -- (output.west);
\end{tikzpicture}\\

$\mathsf{Transit}(k):$\\

\begin{tikzpicture}[
rectnode/.style={rectangle, draw=black, thin,  minimum size=5mm},
layernode/.style={rectangle, draw=black, thin,  minimum size=5mm},
circlenode/.style={circle, draw=black, thin,  minimum size=1mm},
align=center,node distance=0.5cm,
]
\node[]      (input)           {};
\node[rectnode]    (conv)           [right=of input] {$\mathsf{Conv}(1\text{x}1, k)$};
\node[rectnode] (pool)      [right=of conv]      {$\mathsf{AvgPool}(2\times2, 2)$};
\node[]             (output)            [right=of pool]    {};

\draw[->, thick]  (input.east) -- (conv.west);
\draw[->, thick] (conv.east) -- (pool.west);
\draw[->, thick] (pool.east) -- (output.west);
\end{tikzpicture}\\

$\mathsf{Block}(k, r):$\\

\begin{tikzpicture}[
downsamplenode/.style={rectangle, draw=black, thin,  minimum size=5mm},
layernode/.style={rectangle, draw=black, thin,  minimum size=5mm},
circlenode/.style={circle, draw=black, thin,  minimum size=1mm},
align=center,node distance=1cm,
]
\node[]      (input)           {};
\node[] (downsample)      [right=of input]      {};
\node[circlenode]    (layer_1)           [right=of downsample] {$1$};
\node[circlenode]    (layer_2)           [right=of layer_1]    {$2$};
\node[]             (abbr)              [right=of layer_2]     {$\cdots$};
\node[circlenode]    (layer_r)           [right=of abbr]    {$r$};
\node[]             (output)            [right=of layer_r]    {};

\draw[ thick]  (input.east) -- (downsample.east);
\draw[->, thick] (downsample.east) -- (layer_1.west);
\draw[->, thick] (layer_1.east) -- (layer_2.west);
\draw[->, thick] (layer_2.east) -- (abbr.west);
\draw[->,  thick] (abbr.east) -- (layer_r.west);
\draw[->,  thick] (layer_r.east) -- (output.west);
\draw[->, thick] (downsample.east) to[out=30,in=150] (layer_2);
\draw[->, thick] (downsample.east) to[out=35,in=145] (layer_r);
\draw[->, thick] (layer_1.east) to[out=25,in=155] (layer_r);
\draw[->, thick] (layer_2.east) to[out=20,in=160] (layer_r);
\end{tikzpicture} ~~~ where $\circled{$i$}:=\mathsf{Layer}(k)$, joining arrows are channel-wise $\mathsf{concat}$.\\

\item[Forward-Pass:] \hfill\\
$k=32$\\
$\mathsf{Conv}(7\text{x}7, 2k, 2) \rightarrow \mathsf{MaxPool}(2) \rightarrow \mathsf{Block}(k, 6) \rightarrow \mathsf{Transit}(k/2) \rightarrow \mathsf{Block}(k, 12) \rightarrow $\\
$\mathsf{Block}(k, 64) \rightarrow  \mathsf{Transit}(k/2, 48) \rightarrow \mathsf{AvgPool}(\text{global}) \rightarrow \mathsf{FullyConnected}(1000) \rightarrow \mathsf{SoftMax}$

\item[Initialization:] ~~~ $N(0, 2 / \text{fan-in})$, bias $0$\\
\item[Data-Augmentation] \hfill \\ 
scale pixels to $0$ mean  unit variance\\
resize its shorter side to $256$\\
random crop $224\times224$ with horizontal flip \\
SVD $3\times 3$ covariance matrix of RBG pixels over training set: $\lambda_i, \mathsf{v}_i$\\
sample $\alpha_i\sim N(0, 0.01)$ for each image, add $\sum_i \alpha_i \lambda_i \mathsf{v}_i$ to pixels, re-sample every epoch
\item[Training] \hfill\\
batchsize 256 weight-decay 1e-4 momentum 0.9 iteration 60e4 \\
learningrate init 0.1, learningrate $/=10$ every $30$ epochs\\
\item[Testing] \hfill\\
average the logits on the four corners and center $224\times224$ crop with horizontal flip \\

\end{description}


\end{document}